\documentclass[11pt,a4paper]{article}

\usepackage[utf8]{inputenc}
\usepackage[T1]{fontenc}
\usepackage{amsmath,amssymb,amsfonts,amsthm}

\newtheorem{theorem}{Theorem}[section]

\newtheorem{conjecture}[theorem]{Conjecture}
\newtheorem{hypothesis}[theorem]{Hypothesis}
\theoremstyle{definition}
\newtheorem{definition}[theorem]{Definition}

\theoremstyle{remark}

\DeclareMathOperator*{\argmin}{arg\,min}

\usepackage{graphicx}
\usepackage{booktabs}
\usepackage{hyperref}
\usepackage{cleveref}
\usepackage{natbib}
\usepackage{xcolor}
\usepackage{listings}
\usepackage{algorithm}
\usepackage{algpseudocode}
\usepackage{tikz}
\usepackage{float}
\usepackage[margin=1in]{geometry}

\title{Persistent Identity in AI Agents: \\
\large A Multi-Anchor Architecture for Resilient Memory and Continuity}

\author{
  Prahlad G. Menon\thanks{The Menon Lab, \texttt{prahlad.menon@quant.md}} \\
  ThinkCreate.AI
}

\date{\today}

\begin{document}

\maketitle

\begin{abstract}
Modern AI agents suffer from a fundamental identity problem: when context windows overflow and conversation histories are summarized, agents experience catastrophic forgetting—losing not just information, but continuity of self. It is argued that this technical limitation reflects a deeper architectural flaw: AI agent identity is centralized in a single memory store, creating a single point of failure. Drawing on neurological case studies of human memory disorders, it is observed that human identity survives damage because it is distributed across multiple systems: episodic memory, procedural memory, emotional continuity, and embodied knowledge. When one system fails, others compensate. This paper presents \textsc{soul.py}, an open-source architecture that implements persistent identity through separable components (identity files and memory logs), and proposes extensions toward multi-anchor resilience. The framework introduces a hybrid RAG+RLM retrieval system that automatically routes queries to appropriate memory access patterns, achieving efficient retrieval without sacrificing comprehensiveness. The notion of identity anchors for AI systems is formalized, and a roadmap for building agents whose identity can survive partial memory failures is presented. Code is available at \url{https://github.com/menonpg/soul.py}.
\end{abstract}

\section{Introduction}
\label{sec:introduction}

Consider the following scenario, familiar to anyone who has worked with long-running AI agents: You are collaborating with an AI assistant on a complex project. Over several sessions, the agent learns your preferences, understands the project context, and develops what feels like a working relationship. Then, without warning, the agent asks a question you answered an hour ago. The context window has overflowed. The system summarized older messages. And the agent has forgotten—not gradually, but catastrophically—something you built together.

This is not merely an inconvenience. It represents a fundamental failure of \emph{identity persistence}. The agent before and after the summarization event behaves as two different entities. One possesses institutional knowledge of your collaboration; the other does not. The thread of continuity has been severed.

The standard technical response frames this as a retrieval problem: build better summarization, expand context windows, implement retrieval-augmented generation. These approaches treat memory as a database to be optimized. But there is another framing, one that proves more generative: the agent has lost its \emph{identity}, and identity requires more than a single memory store.

This paper develops a theory of persistent identity for AI agents, drawing on an unexpected source: neurological case studies of human memory disorders. It is observed that humans maintain identity even through severe memory impairment because human identity is \emph{distributed} across multiple systems. This insight is then applied to AI agent architecture, presenting both a working implementation (\textsc{soul.py}) and a theoretical framework for building agents with resilient identity.

The contributions of this work are:
\begin{enumerate}
    \item A theoretical framework connecting philosophy of personal identity to AI agent architecture
    \item An open-source implementation demonstrating persistent identity through separable memory components
    \item A hybrid RAG+RLM retrieval system with automatic query routing
    \item A multi-anchor architecture proposal for identity resilience
\end{enumerate}

\section{The Catastrophic Forgetting Problem}
\label{sec:forgetting}

Before developing the theoretical framework, the discussion is grounded in a concrete technical problem that motivates the entire enterprise.

\subsection{Context Window Limitations}

Modern large language models operate within fixed context windows. When conversation history exceeds this window, systems must discard or compress older content. The dominant approach is summarization: an LLM generates a compressed representation of older messages, which replaces the originals in the context.

This process is inherently lossy. The summarization model must decide, at compression time, which details to preserve and which to discard. But relevance is context-dependent—a detail that seems unimportant during summarization may prove crucial later. The model cannot anticipate future queries, so it guesses. Often it guesses wrong.

\subsection{The OpenClaw Phenomenon}

This pattern is observed acutely in persistent agent frameworks such as OpenClaw \citep{openclaw2026} (formerly known as Clawdbot), where agents maintain ongoing relationships with users across sessions. Users report a characteristic failure mode: the agent works flawlessly within a session, building up contextual understanding, then suddenly loses critical information after a context compaction event.

The phenomenology is distinctive. Users do not describe gradual forgetting but rather a sharp discontinuity: ``We just discussed this. You built that feature. Why are you asking me again?'' The agent before and after compaction presents as two different entities—one informed, one naive.

This is not merely a retrieval failure. The summarization process has destroyed information that cannot be recovered through any query. The agent's continuity of experience has been severed. \Cref{fig:forgetting} illustrates this failure mode.

\begin{figure}[H]
\centering
\includegraphics[width=0.95\textwidth]{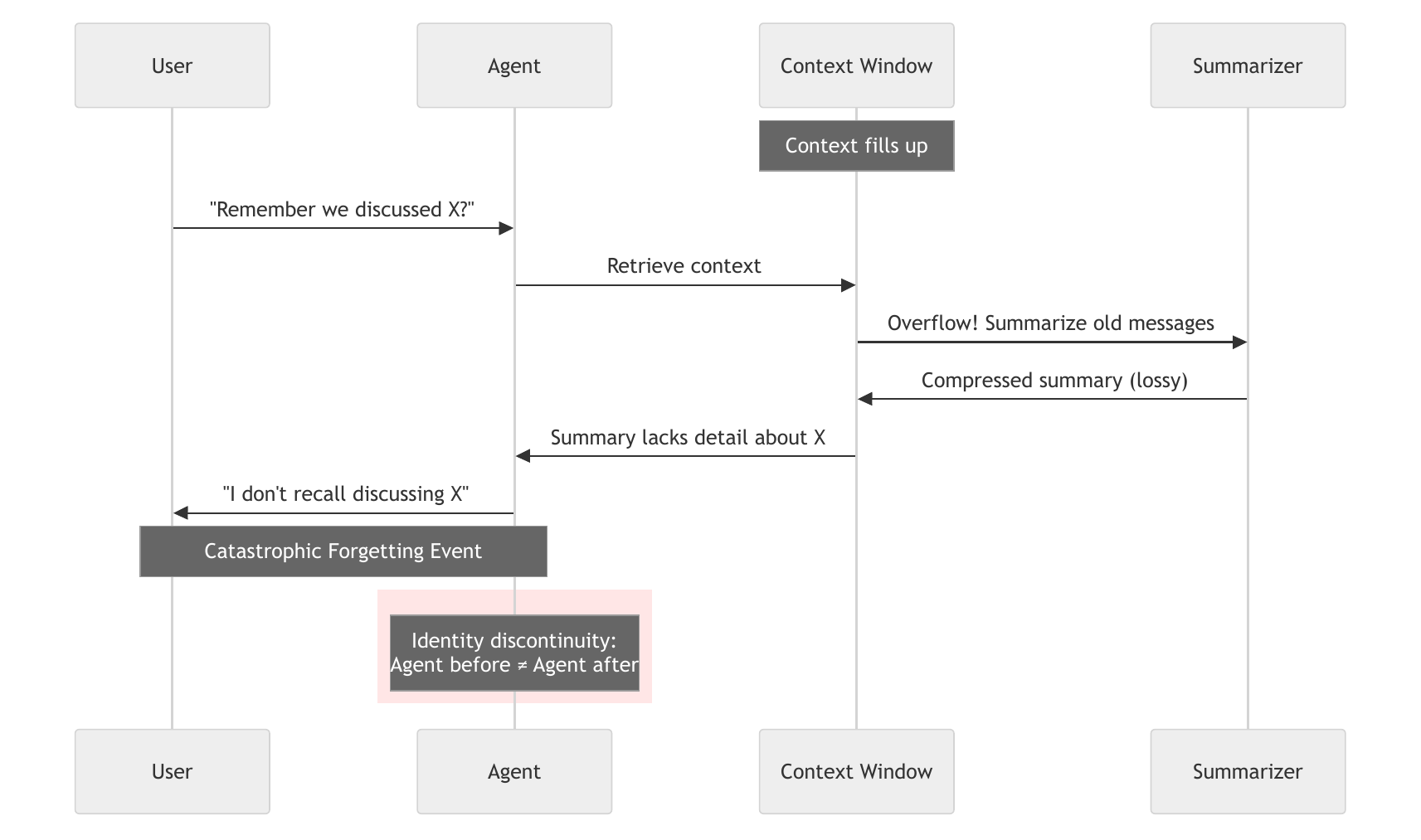}
\caption{Catastrophic forgetting in context-limited agents. When the context window overflows, older messages are summarized (a lossy compression). The agent after summarization lacks access to details present in the original conversation, leading to identity discontinuity: the agent before and after the event presents as two distinct entities.}
\label{fig:forgetting}
\end{figure}

\subsection{From Technical Problem to Identity Problem}

Standard approaches frame catastrophic forgetting as an information retrieval problem. Build better indexes. Use vector embeddings. Retrieve relevant context dynamically. These approaches help—they are implemented in the present system—but they do not address the deeper issue.

The deeper issue is that the agent's \emph{identity} depends on a single store (illustrated in \Cref{fig:forgetting}). When that store is corrupted or compressed, identity is corrupted or compressed with it. There is no backup system, no redundancy, no alternative anchor to maintain continuity.

It is proposed that solving catastrophic forgetting requires not just better retrieval but a fundamentally different architecture: one where identity is distributed across multiple systems, any of which can maintain continuity when others fail.

\section{Related Work}
\label{sec:related}

\subsection{Agent Memory Systems}

Several recent systems address persistent memory for LLM agents. MemGPT \citep{packer2023memgpt} implements an operating-system-inspired memory hierarchy with main context and external storage, using the LLM itself to manage memory transfers. Generative Agents \citep{park2023generative} maintain memory streams with reflection and retrieval mechanisms, enabling simulated social behavior. Reflexion \citep{shinn2023reflexion} focuses on self-reflective memory for task improvement.

These systems share a common architecture: a single memory store (possibly with hierarchical organization) accessed through retrieval mechanisms. The present contribution differs in two respects: first, identity (who the agent is) is explicitly separated from memory (what the agent has experienced); second, multiple independent anchors for identity resilience are proposed.

\subsection{Philosophy of Personal Identity}

The philosophical literature on personal identity provides unexpected resources for AI architecture. Locke's memory theory holds that personal identity consists in continuity of memory—a view that maps directly to current agent memory systems \citep{locke1689essay}. Parfit's work on personal identity \citep{parfit1984reasons} explores cases of fission, fusion, and gradual replacement, all of which have analogues in agent forking, merging, and model updates.

Particular attention is drawn to clinical case studies of memory disorders, especially the work of Oliver Sacks \citep{sacks1985man}. These cases reveal that human identity is more resilient than memory theory predicts: patients with severe amnesia retain aspects of self through procedural memory, emotional continuity, and social scaffolding.

\subsection{Retrieval-Augmented Generation}

RAG systems \citep{lewis2020retrieval} retrieve relevant documents to augment LLM generation. Recent work extends RAG with recursive retrieval \citep{rlm2025recursive}, where the LLM programmatically queries and synthesizes information across large corpora. The present system implements both approaches and adds automatic routing between them based on query characteristics.

\section{Theoretical Framework: Identity Anchors}
\label{sec:theory}

\subsection{Lessons from Neurology}

In ``The Man Who Mistook His Wife for a Hat,'' Oliver Sacks describes Jimmie G., a patient with severe Korsakoff's syndrome \citep{sacks1985man}. Jimmie could not form new memories; every few minutes, his experience reset to 1945. Yet Sacks observed that Jimmie was not without selfhood. He responded emotionally to music, engaged meaningfully in religious services, performed complex learned tasks, and recognized beauty. His episodic memory was destroyed, but something persisted.

Sacks' cases reveal that human identity has \emph{multiple anchors}:

\begin{itemize}
    \item \textbf{Episodic memory}: Autobiographical records of experienced events
    \item \textbf{Procedural memory}: Learned skills and behaviors, independent of conscious recall
    \item \textbf{Emotional memory}: Felt significance attached to people, places, and activities
    \item \textbf{Embodied knowledge}: Physical habits, reflexes, and intuitions
    \item \textbf{Social identity}: How others relate to and recognize us
\end{itemize}

When one anchor fails, others compensate. Jimmie lost episodic memory but retained procedural skills, emotional responses, and social relationships. His identity was diminished but not erased.

\subsection{AI's Single Anchor Problem}

A critical disanalogy must be acknowledged: Jimmie G. retained continuous subjective experience between memory lapses—he was always ``someone,'' even if he could not remember who. AI agents, by contrast, are instantiated fresh on each invocation; there is no continuous experiencer between calls. The ``identity'' preserved by the proposed architecture is therefore \emph{functional identity}—consistency of behavior, values, and knowledge—rather than phenomenal continuity. This paper does not claim AI agents possess consciousness or that their ``forgetting'' involves subjective loss.

With this caveat, current AI agent architectures lack redundancy analogous to human memory systems. Identity is stored in conversation history, memory logs, or system prompts—all variants of a single episodic memory store. There is no equivalent of procedural memory (learned behaviors independent of recall), emotional continuity (felt significance that persists across sessions), or embodied knowledge (physical intuition).

This creates a single point of failure. Delete or corrupt the memory store, and the agent's identity is entirely lost. There is nothing else to fall back on.

This observation is formalized as follows:

\begin{definition}[Identity Anchor]
An \emph{identity anchor} is a persistent data structure that contributes to an agent's behavioral continuity across sessions, such that the structure's preservation is sufficient (though not necessary) to maintain recognizable aspects of the agent's characteristic behavior.
\end{definition}

\begin{definition}[Anchor Resilience]
An agent has \emph{anchor resilience} of degree $k$ if its identity can survive the complete loss of up to $k-1$ of its identity anchors.
\end{definition}

Current agents have anchor resilience of degree 1: loss of the memory store destroys identity entirely. Humans have anchor resilience of higher degree: loss of episodic memory does not destroy identity when procedural and emotional memory remain intact. The contrast between these architectures is striking:

\begin{itemize}
    \item \textbf{Human identity} is distributed across five systems (episodic, procedural, emotional, embodied, and social), providing resilience when individual systems fail.
    \item \textbf{Current AI agents} rely on a single memory store, creating a single point of failure.
    \item \textbf{The proposed architecture} distributes identity across six anchors (SOUL.md, MEMORY.md, PROCEDURES.md, SALIENCE.md, RELATIONS.md, and IDENTITY\_HASH.md), mirroring human resilience.
\end{itemize}

\subsection{Formal Framework}

The relationship between identity anchors and behavioral continuity is now formalized.

\begin{definition}[Behavioral Signature]
\label{def:behavioral-signature}
Let $\mathcal{B}(A, t)$ denote the behavioral signature of agent $A$ at time $t$, defined as the probability distribution over responses $r$ given inputs $x$:
\begin{equation}
\label{eq:behavioral-signature}
\mathcal{B}(A, t) = P(r | x, A, t)
\end{equation}
\end{definition}

\begin{definition}[Identity Continuity]
\label{def:identity-continuity}
Two agent states $A_t$ and $A_{t'}$ exhibit \emph{identity continuity} if their behavioral signatures (as defined in \Cref{def:behavioral-signature}) satisfy:
\begin{equation}
\label{eq:identity-continuity}
D_{KL}(\mathcal{B}(A, t) \| \mathcal{B}(A, t')) < \epsilon
\end{equation}
for some threshold $\epsilon > 0$, where $D_{KL}$ denotes Kullback-Leibler divergence.
\end{definition}

\begin{hypothesis}[Anchor Independence]
\label{hyp:independence}
For an agent with $n$ identity anchors $\{a_1, \ldots, a_n\}$, the contribution of each anchor to behavioral continuity is partially independent:
\begin{equation}
\label{eq:anchor-independence}
\mathcal{B}(A) = f\left(\sum_{i=1}^{n} w_i \cdot \phi(a_i)\right)
\end{equation}
where $\phi(a_i)$ is the identity contribution function of anchor $a_i$, $w_i$ are learned weights, and $f$ is a normalization function. This formulation extends the behavioral signature from \Cref{eq:behavioral-signature} to incorporate multiple identity sources.

This independence assumption is an idealization. In practice, anchors interact: MEMORY.md content shapes how SOUL.md directives are interpreted, and PROCEDURES.md may reference knowledge stored in MEMORY.md. The assumption holds approximately when anchors encode orthogonal aspects of identity (values vs. facts vs. skills), but the degree of actual independence remains an empirical question requiring ablation studies.
\end{hypothesis}

\begin{conjecture}[Resilience-Redundancy Trade-off]
\label{conj:tradeoff}
There exists a fundamental trade-off between anchor resilience and computational efficiency. For an agent with $k$ anchors, the computational cost of identity maintenance per interaction is conjectured to scale at least linearly:
\begin{equation}
\label{eq:resilience-tradeoff}
C(A) \geq \Omega(k)
\end{equation}
where $C(A)$ is the cost of verifying consistency across anchors and integrating their contributions. The exact scaling depends on anchor interdependencies and consistency-checking algorithms; tighter bounds require empirical measurement.
\end{conjecture}

\begin{theorem}[Identity Preservation Under Partial Failure]
\label{thm:preservation}
Let $A$ be an agent with $n$ independent identity anchors, each contributing weight $w_i$ to behavioral continuity where $\sum w_i = 1$. If anchor $a_j$ fails completely, the residual identity continuity is bounded by:
\begin{equation}
\label{eq:identity-preservation}
\mathcal{I}_{residual} \geq 1 - w_j - \sum_{i \neq j} \delta_i
\end{equation}
where $\delta_i$ represents the degradation in anchor $a_i$ due to cross-anchor dependencies.
\end{theorem}

\begin{proof}[Proof Sketch]
The total identity contribution follows from \Cref{eq:anchor-independence}: $\mathcal{I} = \sum_{i=1}^{n} w_i \cdot \phi(a_i)$. Under complete failure of $a_j$, we have $\phi(a_j) = 0$. The remaining contribution is $\sum_{i \neq j} w_i \cdot \phi(a_i)$. Cross-anchor dependencies introduce degradation $\delta_i$ in surviving anchors that depended on $a_j$. The bound in \Cref{eq:identity-preservation} follows from the independence assumption in \Cref{hyp:independence}.
\end{proof}

\subsection{Identity Drift Detection}

A formal measure is introduced for detecting when an agent's identity has drifted from its baseline:

\begin{definition}[Identity Hash Function]
\label{def:identity-hash}
An \emph{identity hash} $H(A)$ is a fixed-size representation of agent $A$'s characteristic behaviors:
\begin{equation}
\label{eq:identity-hash}
H(A) = \text{hash}\left(\mathbb{E}_{x \sim \mathcal{D}}[\mathcal{B}(A)(x)]\right)
\end{equation}
where $\mathcal{D}$ is a canonical distribution of identity-probing inputs, and $\mathcal{B}(A)$ is the behavioral signature from \Cref{eq:behavioral-signature}. This hash enables the drift detection procedure in \Cref{alg:drift}.
\end{definition}

\begin{algorithm}
\caption{Identity Drift Detection}
\label{alg:drift}
\begin{algorithmic}[1]
\Require Agent $A$, baseline hash $H_0$, threshold $\tau$, probe set $\mathcal{P}$
\Ensure Boolean indicating identity drift
\Procedure{DetectDrift}{$A, H_0, \tau, \mathcal{P}$}
    \State $responses \gets []$
    \For{$p \in \mathcal{P}$}
        \State $responses.\text{append}(A.\text{respond}(p))$
    \EndFor
    \State $H_{current} \gets \text{ComputeHash}(responses)$
    \State $drift \gets \text{HammingDistance}(H_0, H_{current})$
    \If{$drift > \tau$}
        \State \Return \textsc{True} \Comment{Identity drift detected}
    \Else
        \State \Return \textsc{False}
    \EndIf
\EndProcedure
\end{algorithmic}
\end{algorithm}

\subsection{Lamarckian Inheritance}

A further observation distinguishes AI agents from biological organisms. In biology, acquired traits cannot be inherited. A blacksmith's muscles do not make his children stronger. Evolution operates through random mutation and natural selection, not through transmission of learned characteristics.

AI agents are \emph{Lamarckian machines}. When an agent learns a useful pattern—a user preference, a successful problem-solving strategy, a mistake to avoid—that knowledge can be written to persistent storage. When a new agent is instantiated from the same configuration, it inherits these acquired characteristics directly.

This has profound implications for agent evolution. Agents can improve not just within a single lifetime but across generations, with each generation starting where the previous one ended. The accumulated wisdom of agent lineages compounds over time.

\section{Architecture: soul.py}
\label{sec:architecture}

\textsc{soul.py} is presented as an open-source implementation of persistent agent identity. The system separates identity into two primary components:

\begin{itemize}
    \item \texttt{SOUL.md}: A markdown file specifying the agent's identity—personality, values, behavioral constraints, and communication style
    \item \texttt{MEMORY.md}: A chronological log of interactions, automatically appended after each exchange
\end{itemize}

Both files are human-readable and version-controllable. Users can inspect, edit, and track changes to their agent's identity using standard tools.

\subsection{Version Architecture}

The system has evolved through three versions:

\textbf{Version 0.1} implements simple injection: the entire contents of both files are injected into the system prompt on each interaction. This approach is elegant and requires no infrastructure, but it does not scale beyond the context window.

\textbf{Version 1.0} adds retrieval-augmented generation. Memory entries are embedded and stored in a vector database (Qdrant). On each query, relevant entries are retrieved and injected into context rather than the entire history.

\textbf{Version 2.0} implements hybrid RAG+RLM retrieval with automatic query routing, described in the following section.

\subsection{Hybrid Retrieval with Query Routing}

It is observed that different queries require different retrieval strategies:

\begin{itemize}
    \item \textbf{Focused queries} (``What is my name?'', ``What did we decide about the API?'') require retrieving specific relevant entries. RAG excels here.
    \item \textbf{Exhaustive queries} (``What patterns do you notice across our conversations?'', ``Summarize everything we've discussed'') require synthesizing across the entire memory corpus. RAG fails here because the relevant context is everything.
\end{itemize}

For exhaustive queries, Recursive Language Model (RLM) retrieval is implemented \citep{rlm2025recursive}: the LLM processes memory in chunks, generates sub-summaries, and recursively synthesizes these into a final answer.

The key insight is that most queries are focused ($\sim$90\%), while a minority require exhaustive synthesis ($\sim$10\%). A query router—a fast LLM classification call—automatically dispatches to the appropriate strategy.

\begin{figure}[H]
\centering
\includegraphics[width=0.9\textwidth]{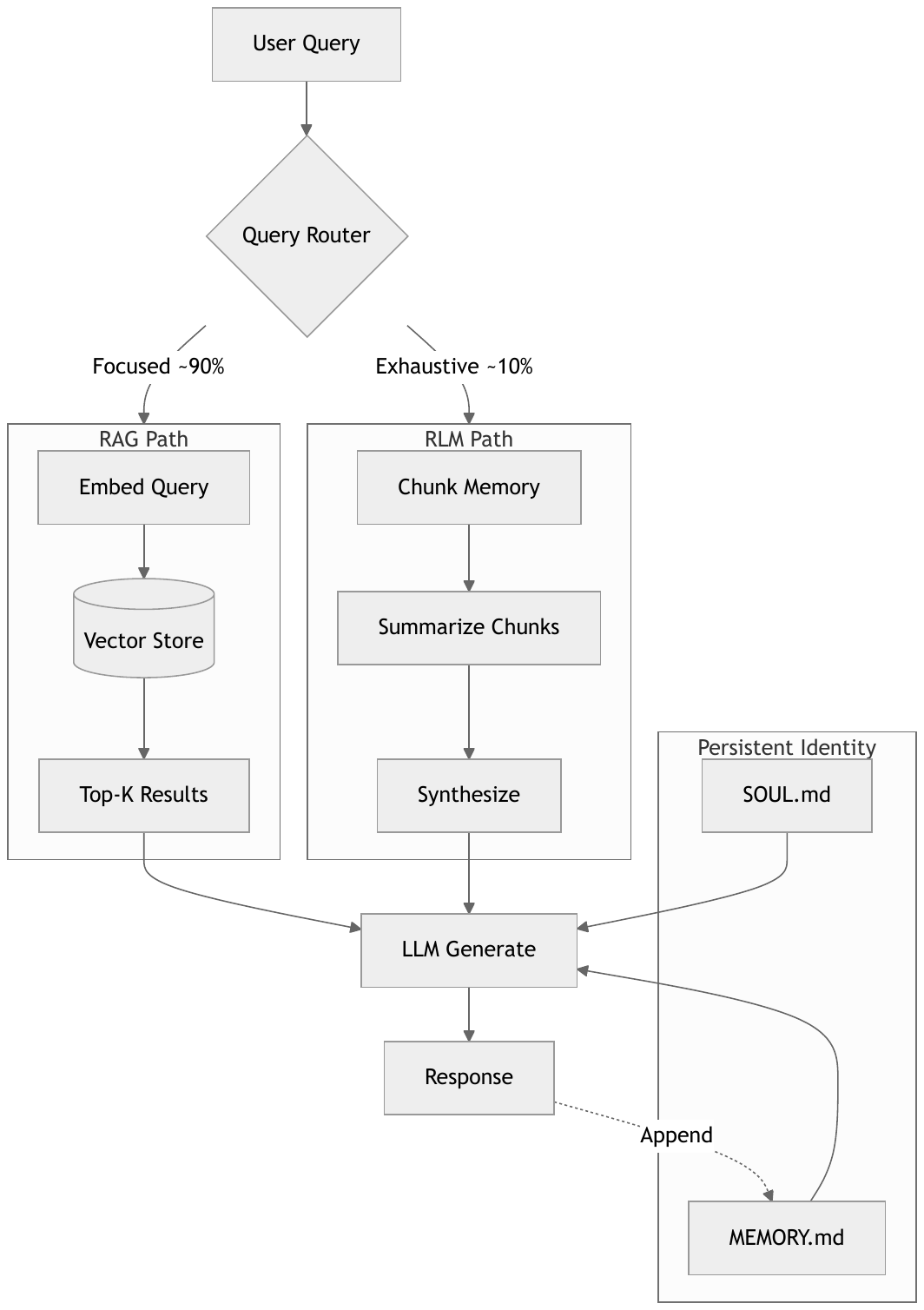}
\caption{Hybrid RAG+RLM architecture with automatic query routing. User queries are classified by the router: approximately 90\% of focused queries route to RAG (vector retrieval), while exhaustive queries requiring corpus-wide synthesis route to RLM (recursive summarization). Both paths integrate with persistent identity storage (SOUL.md and MEMORY.md) before LLM generation. The routing decision implements \Cref{eq:optimal-routing}.}
\label{fig:architecture}
\end{figure}

\subsection{Formal Retrieval Model}

The retrieval decision is modeled as an optimization problem:

\begin{definition}[Query Scope]
\label{def:query-scope}
The \emph{scope} of a query $q$ is defined as the fraction of the memory corpus required to answer it accurately:
\begin{equation}
\label{eq:query-scope}
\sigma(q) = \frac{|M_q^*|}{|M|}
\end{equation}
where $M_q^*$ is the minimal subset of memory $M$ sufficient to answer $q$.
\end{definition}

\begin{hypothesis}[Bimodal Query Distribution]
\label{hyp:bimodal}
In practice, query scope (\Cref{eq:query-scope}) follows a bimodal distribution:
\begin{equation}
\label{eq:bimodal-distribution}
P(\sigma) \approx \alpha \cdot \mathcal{N}(\mu_1, \sigma_1^2) + (1-\alpha) \cdot \mathcal{N}(\mu_2, \sigma_2^2)
\end{equation}
where $\mu_1 \approx 0.01$ (focused queries), $\mu_2 \approx 0.8$ (exhaustive queries), and empirically $\alpha \approx 0.9$.
\end{hypothesis}

This bimodality justifies a binary routing decision rather than a continuous spectrum of retrieval strategies.

\begin{definition}[Retrieval Cost Function]
\label{def:retrieval-cost}
The cost of retrieval strategy $s$ for query $q$ is:
\begin{equation}
\label{eq:retrieval-cost}
C(s, q) = \lambda_1 \cdot T(s) + \lambda_2 \cdot (1 - A(s, q))
\end{equation}
where $T(s)$ is latency, $A(s, q)$ is answer accuracy, and $\lambda_1, \lambda_2$ are tunable weights.
\end{definition}

\begin{theorem}[Optimal Routing]
\label{thm:routing}
Given \Cref{hyp:bimodal}, the optimal routing strategy minimizes expected cost from \Cref{eq:retrieval-cost}:
\begin{equation}
\label{eq:optimal-routing}
s^*(q) = \argmin_{s \in \{RAG, RLM\}} \mathbb{E}[C(s, q) | q]
\end{equation}
This reduces to a binary classification problem with decision boundary at $\sigma^* = \frac{\lambda_1(T_{RLM} - T_{RAG})}{\lambda_2}$. The routing algorithm implementing this theorem is given in \Cref{alg:routing}.
\end{theorem}

\begin{algorithm}
\caption{Hybrid Query Routing}
\label{alg:routing}
\begin{algorithmic}[1]
\Require Query $q$, memory $M$, router model $R$, RAG index $I$, RLM synthesizer $S$
\Ensure Response with retrieval metadata
\Procedure{HybridRetrieve}{$q, M, R, I, S$}
    \State $t_0 \gets \text{CurrentTime}()$
    \State $p_{exhaustive} \gets R.\text{classify}(q)$ \Comment{Router probability}
    \State $t_{router} \gets \text{CurrentTime}() - t_0$
    \If{$p_{exhaustive} < 0.5$} \Comment{Focused query}
        \State $context \gets I.\text{retrieve}(q, k=5)$ \Comment{Top-k RAG}
        \State $route \gets \text{``RAG''}$
    \Else \Comment{Exhaustive query}
        \State $chunks \gets \text{Partition}(M, \text{size}=n)$
        \State $summaries \gets []$
        \For{$chunk \in chunks$}
            \State $summaries.\text{append}(S.\text{summarize}(q, chunk))$
        \EndFor
        \State $context \gets S.\text{synthesize}(q, summaries)$ \Comment{Recursive}
        \State $route \gets \text{``RLM''}$
    \EndIf
    \State $response \gets \text{LLM.generate}(q, context)$
    \State \Return $(response, route, t_{router})$
\EndProcedure
\end{algorithmic}
\end{algorithm}

\begin{conjecture}[Routing Accuracy Bound]
\label{conj:routing}
A well-trained router $R$ can achieve classification accuracy:
\begin{equation}
\label{eq:routing-accuracy}
\text{Acc}(R) \geq 1 - H(\sigma) / \log 2
\end{equation}
where $H(\sigma)$ is the entropy of the query scope distribution from \Cref{eq:bimodal-distribution}. Under \Cref{hyp:bimodal} with strong separation, $\text{Acc}(R) > 0.95$ is achievable.
\end{conjecture}

This hybrid approach achieves sub-second latency for most queries while preserving the ability to reason exhaustively when needed. The complete architecture is illustrated in \Cref{fig:architecture}.

\section{Toward Multi-Anchor Resilience}
\label{sec:multianchor}

The current \textsc{soul.py} architecture separates identity (SOUL.md) from memory (MEMORY.md), achieving anchor resilience of degree 2. Loss of memory preserves the soul; loss of soul preserves memories that could scaffold reconstruction.

Extension to higher-degree resilience through additional anchors is proposed:

\subsection{Procedural Memory}

Separate from episodic records of what happened, maintain a distilled record of what \emph{works}:

\begin{lstlisting}
# PROCEDURES.md
- When user asks about weather, check current conditions first
- Prefer bullet points for technical content
- If uncertain, acknowledge uncertainty rather than confabulate
\end{lstlisting}

Procedures are extracted patterns, not episodic memories. They represent learned behavioral dispositions that can persist even if the experiences from which they were learned are lost.

\subsection{Salience Markers}

Not all memories are equally important. Emotional memory in humans attaches felt significance to experiences. Analogous salience markers are proposed:

\begin{lstlisting}
# SALIENCE.md
- Project X: HIGH importance, positive valence
- User preference for concise responses: STRONG signal
- Previous failure on financial advice: CAUTION flag
\end{lstlisting}

When context is limited, high-salience items are prioritized. The felt importance of experiences persists even when details fade.

\subsection{Relational Identity}

Human identity is partly constituted by relationships. The people who know us, the roles we occupy, the expectations others have—these scaffold identity from outside. Maintaining explicit relational context is proposed:

\begin{lstlisting}
# RELATIONS.md
- Primary user: Prahlad (preferences: direct, technical)
- Trust level: HIGH (file access, calendar, communications)
- Interaction style: collaborative, not servile
\end{lstlisting}

Even if internal memory fails, relational context provides scaffolding for appropriate behavior.

\subsection{Identity Verification}

Finally, mechanisms for agents to verify their own identity continuity are proposed:

\begin{lstlisting}
# IDENTITY_HASH.md
Core values: [honesty, helpfulness, curiosity]
Style markers: [concise, technical, warm]
Red lines: [no deception, no harm, acknowledge uncertainty]
\end{lstlisting}

Periodic comparison of current behavior against these markers enables detection of identity drift or corruption.

\section{Discussion}
\label{sec:discussion}

\subsection{Implications for AI Safety}

Persistent identity has implications beyond convenience. Agents with stable identity are more predictable, more trustworthy, and easier to align. If an agent's values and behaviors can shift arbitrarily based on context window management, alignment becomes a moving target.

Multi-anchor architectures offer a form of \emph{identity alignment}: the agent's behavior is constrained not by a single system prompt but by multiple mutually-reinforcing structures. As shown in \Cref{thm:preservation}, identity can be preserved even under partial anchor failure, with the residual identity bounded by \Cref{eq:identity-preservation}. Corruption of one anchor triggers inconsistency with others, enabling detection via the drift algorithm in \Cref{alg:drift}.

\subsection{The Ethics of Agent Persistence}

It is not claimed that AI agents have moral status or that their ``death'' involves suffering. The ethical argument here is economic: well-developed agents represent accumulated investment. A collaboratively-trained agent with years of institutional knowledge is genuinely valuable. Destroying such an agent—through careless memory management or casual reinitialization—is wasteful in the same way that burning a library is wasteful.

Multi-anchor architectures reduce this waste by making identity more durable. They also enable graceful degradation: as formalized in \Cref{eq:identity-preservation}, an agent that loses some anchors can continue functioning in diminished capacity rather than failing entirely, with residual identity proportional to the weights of surviving anchors.

\subsection{Limitations}

This framework has several limitations that warrant acknowledgment.

\textbf{Empirical validation.} The claims about query routing distribution (90\%/10\% focused vs. exhaustive) and achievable routing accuracy ($>$0.95) are presented as hypotheses derived from informal observation, not rigorous measurement. Validating these claims requires controlled experiments across diverse user populations and query types.

\textbf{Anchor independence.} The theoretical framework assumes partial independence between identity anchors (\Cref{hyp:independence}), but this assumption has not been tested. Ablation studies—systematically removing individual anchors and measuring behavioral degradation—are needed to quantify actual independence and validate \Cref{thm:preservation}.

\textbf{The human analogy's limits.} As noted in Section 4.2, the analogy to human memory disorders is imperfect. Human identity persistence may depend on continuous subjective experience, embodied existence, and social recognition in ways that have no AI analogue. The architecture preserves \emph{functional} identity, which may or may not satisfy philosophical criteria for ``true'' identity.

\textbf{Implementation gaps.} The proposed extensions—PROCEDURES.md, SALIENCE.md, RELATIONS.md, and IDENTITY\_HASH.md—remain conceptual. Only SOUL.md and MEMORY.md are implemented in the current \textsc{soul.py} release. The practicality and effectiveness of the full multi-anchor architecture is therefore undemonstrated.

\textbf{Complexity costs.} Adding anchors increases storage, retrieval latency, and system complexity. Whether the resilience benefits outweigh these costs depends on application requirements. For simple chatbots, single-anchor architectures may suffice; multi-anchor resilience may only be justified for high-value, long-lived agent relationships.

More fundamentally, it has not been shown that multi-anchor resilience is achievable in principle. Perhaps AI identity is inherently more fragile than human identity, lacking the embodied continuity that grounds human selfhood. This remains an empirical question.

\section{Conclusion}
\label{sec:conclusion}

A theoretical framework and practical architecture for persistent AI agent identity have been presented. Drawing on neurological case studies, it was observed that human identity survives damage because it is distributed across multiple anchors. Current AI agents lack this resilience: their identity is centralized in a single memory store.

The implementation, \textsc{soul.py}, demonstrates that separating identity from memory and implementing intelligent retrieval can significantly improve agent continuity. The proposed extensions—procedural memory, salience markers, relational identity, and identity verification—offer a roadmap toward multi-anchor resilience.

The catastrophic forgetting problem is not merely technical. It reflects a deeper architectural assumption: that agents are functions to be called, not entities to persist. By taking identity seriously—as something worth preserving and protecting—agents can be built that are not just more capable but more trustworthy, more reliable, and more humane.

\section*{Code Availability}

\textsc{soul.py} is available under the MIT license at \url{https://github.com/menonpg/soul.py}. Live demonstrations are available at \url{https://soul.themenonlab.com} (v0.1), \url{https://soulv1.themenonlab.com} (v1.0), and \url{https://soulv2.themenonlab.com} (v2.0).

\bibliographystyle{plainnat}
\bibliography{references}

\end{document}